\documentclass[nohyperref]{article}

\usepackage{microtype}
\usepackage{graphicx}
\usepackage{subfigure}
\usepackage{booktabs} 

\usepackage{hyperref}

\usepackage[accepted]{icml2022_pods}

\usepackage{amsmath}
\usepackage{amssymb}
\usepackage{mathtools}
\usepackage{amsthm}

\usepackage[capitalize,noabbrev]{cleveref}

\theoremstyle{plain}
\newtheorem{theorem}{Theorem}[section]

\newtheorem{lemma}[theorem]{Lemma}

\theoremstyle{definition}
\newtheorem{definition}[theorem]{Definition}

\theoremstyle{remark}

\usepackage[textsize=tiny]{todonotes}

\usepackage[resetlabels]{multibib}
\usepackage{multirow}
\usepackage{multicol}
\usepackage{enumitem}
\usepackage[normalem]{ulem}

\usepackage{xcolor}
\usepackage{threeparttable}

\definecolor{cite_color}{HTML}{114083}
\definecolor{link_color}{RGB}{153, 0,0}  
\definecolor{url_color}{RGB}{153, 102,  0}
\definecolor{emp_color}{RGB}{0,0,255}
\hypersetup{
 colorlinks,
 citecolor=cite_color,
 linkcolor=link_color,
 urlcolor=url_color}
 
\usepackage{caption}
\usepackage{pifont}
\usepackage{makecell}
\usepackage{graphicx}
\usepackage{subfigure} 
\usepackage{amsmath}
\usepackage{multirow}
\usepackage{multicol}
\usepackage{enumitem}
\usepackage{ulem}
\usepackage{threeparttable}
\usepackage{wrapfig}
\usepackage{scrextend}
\usepackage{thm-restate}
\usepackage{array,hhline}
\usepackage{amsmath,amsfonts}
\usepackage{graphicx}
\usepackage{tabularborder}
\usepackage{etoc}
\etocdepthtag.toc{mtchapter}
\etocsettagdepth{mtchapter}{subsection}
\etocsettagdepth{mtappendix}{none}
\usepackage{framed} 
\definecolor{shadecolor}{rgb}{0.94, 0.97, 1.0}
\newcommand{\jy}[1]{\textcolor{red}{}}

\newcommand{\Eqref}[1]{Eq.~\eqref{#1}}
\newcommand{\mbf}[1]{\mathbf{#1}}
\newcommand{\mbb}[1]{\mathbb{#1}}
\newcommand{\mcal}[1]{\mathcal{#1}}
\usepackage{xspace} 
\newcommand{\ConfirmationBias}{confidence bias\xspace}

\usepackage{amssymb} 
\def\R{\mathbb{R}}

\def\x{\mathbf{x}}

\def\z{\mathbf{z}}
\def\xx{\times}
\def\V{\mathcal{V}}

\def\G{\mathcal{G}}

\def\F{\mathcal{F}}

\icmltitlerunning{GraphTTA: Test Time Adaptation on Graph Neural Networks}

\begin{document}

\twocolumn[
\icmltitle{GraphTTA: Test Time Adaptation on Graph Neural Networks}



\icmlsetsymbol{equal}{*}

\begin{icmlauthorlist}
\icmlauthor{Guanzi Chen}{equal,yyy}
\icmlauthor{Jiying Zhang}{equal,xxx}
\icmlauthor{Xi Xiao}{xxx}
\icmlauthor{Yang Li}{yyy}
\end{icmlauthorlist}

\icmlaffiliation{yyy}{Tsinghua-Berkeley Shenzhen Institute, Tsinghua Universtiy, Shenzhen, China}
\icmlaffiliation{xxx}{Shenzhen International Graduate School, Tsinghua University, Shenzhen, China}

\icmlcorrespondingauthor{Yang Li}{yangli@sz.tsinghua.edu.cn}
\icmlcorrespondingauthor{Xi Xiao}{xiaox@sz.tsinghua.edu.cn} 

\icmlkeywords{Machine Learning, ICML}

\vskip 0.3in
]



\printAffiliationsAndNotice{\icmlEqualContribution} 

\begin{abstract}
Recently, test time adaptation~(TTA) has attracted increasing attention due to its power of handling the distribution shift issue in the real world.
Unlike what has been developed for convolutional neural networks (CNNs) for image data, TTA is less explored for Graph Neural Networks (GNNs). 
There is still a lack of efficient algorithms tailored for graphs with irregular structures. 
In this paper, we present a novel test time adaptation strategy named Graph Adversarial Pseudo Group Contrast~(GAPGC), for graph neural networks TTA, to better adapt to the Out Of  Distribution~(OOD) test data. 
Specifically, GAPGC employs a contrastive learning variant as a self-supervised task during TTA, equipped with Adversarial Learnable Augmenter and Group Pseudo-Positive Samples to 
enhance the relevance between the self-supervised task and the main task, boosting the performance of the main task. 
Furthermore, we provide theoretical evidence that GAPGC can extract minimal sufficient information for the main task from information theory perspective.
Extensive experiments on molecular scaffold OOD dataset demonstrated that the proposed approach achieves state-of-the-art performance on GNNs.
\end{abstract}

\section{Introduction}
\begin{figure*}[th]
    \centering
          \vspace{-0.5mm}
    \includegraphics[width=0.99\linewidth]{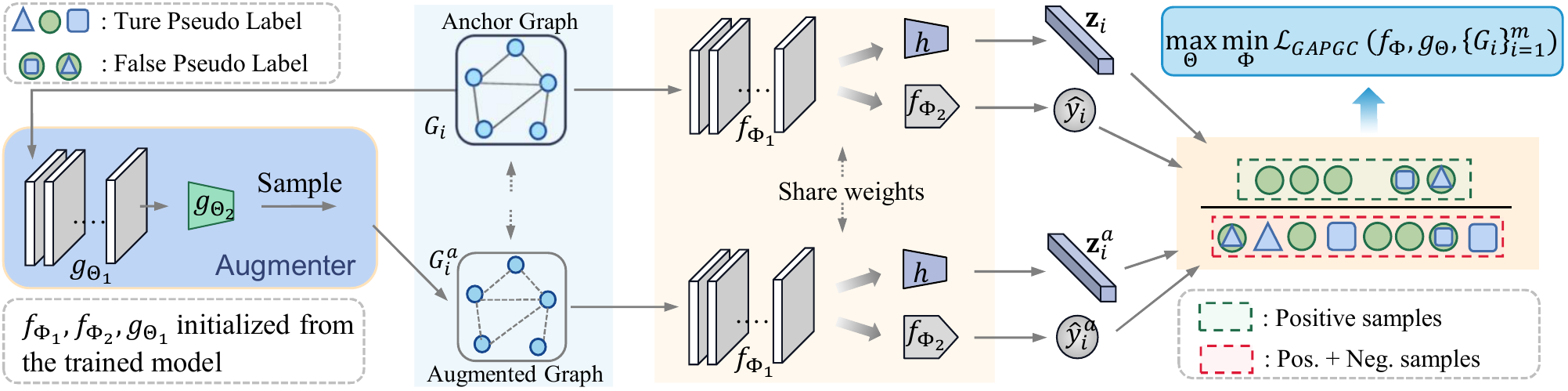}
    \vspace{-1mm}
     \caption{Overall framework of GAPGC for graph test time adaptation. The circle, triangle and square on the left  denote three different classes.  The green circles represent the selected positive samples with the same pseudo-class as the anchor graph.}
    \label{fig:framework}
        \vspace{-4mm}
\end{figure*}

Deep neural network models have gained excellent performance under the condition that training and testing data from the same distribution~\cite{kipf2016semi,xu2018powerful, zhang2022learnable}. However, performance suffers when the training data differ from the test data (also called \textit{distribution shift})~\cite{ding2021closer,li2022out}.

Recently test time adaptation has been proposed to improve the performance of OOD test data via model adaptation with test samples and has shown superior performance in the visual domain.
For example, Tent~\cite{wang2020tent} uses entropy minimization as a self-supervised objective during testing,
adapting the model by minimizing the entropy of the model predictions on test samples,
and TTT~\cite{sun2020test} introduces a rotation task as a self-supervised auxiliary task to be jointly optimized with the main task (the target task) during training and to further finetune the trained model during testing, while  TTT++~\cite{liu2021ttt++} replaces the rotation task with a contrastive learning task, which can extract discriminative representation by bringing closer the similar instances and separating the dissimilar instances. Despite the encouraging progress, existing TTA schemes are focusing on image data.
However, there is still a lack of efficient TTA algorithms tailored for graph data, for which there also exist many OOD circumstances, including molecular property prediction with different molecular scaffolds and protein fold classification with various protein families~\cite{2022arXiv220109637J,li2022out}, etc.

It is imperative to investigate the test time adaptation method tailored for GNNs for the following limitations of the current methods:
1) Unlike images, graph data is irregular and abstract of diverse nature (e.g. citation networks, social networks, and biomedical networks,
~\cite{zhang2022hypergraph}),
thus image-based methods are not suitable for graphs. For instance, many image augmentation-based methods~\cite{ashukha2020pitfalls,zhang2021memo}
can not be simply extended to graph due to the complication of graph augmentations.
2) Most of data-agnostic approaches involve Entropy Minimization~\cite{wang2020tent,niu2022efficient}, 
which is in effect equivalent to pseudo label~\cite{lee2013pseudo}. Trusting false pseudo-labels as “ground truth” by encoding them as hard labels could lead to overconfident mistakes (\textit{confidence bias}) and easily confuse the models \cite{zou2019confidence}.
3) The GNN encoders may learn representations involving the label-irrelevant redundant information through the self-supervised auxiliary task such as contrastive learning~\cite{liu2021ttt++}, resulting in a sub-optimal performance in the main task~\cite{suresh2021adversarial}.

To tackle above limitations, we propose Graph Adversarial Pseudo Group Contrast (namely GAPGC), a novel TTA method tailored for GNNs with a contrastive loss variant as the self-supervised objective during testing.  
Firstly, GAPGC uses group pseudo-positive samples, i.e. a group of graph augmentations with the same class as the anchor graph are selected as the positive samples~\citep{wang2021self}, where the class information comes from the pseudo-labels output by the online test model initialized with the offline trained model.
The group pseudo-positive samples enable the contrastive loss to exploit pseudo-labels into model training, bringing useful discriminative information from the well-trained model to the embedding space, thereby enhancing the relevance between the contrastive self-supervised task and the main task. Meanwhile, it also can mitigate the reliance on pseudo-labels and boost the tolerance to incorrect pseudo-labels. This is especially important for label-sensitive graph augmentations. 

Secondly, GAPGC adopts adversarial learning to avoid capturing redundant information. Specifically, a Adversarial Learnable Augmenter is proposed to generate aggressive positive samples.
On the one hand, GAPGC enforces the augmenter to disturb the original graphs and decrease the information being encoded by maximizing the contrastive loss. On the other hand, GAPGC optimizes the encoder to maximize the correspondence between the disturbed graph pairs by minimizing the contrastive loss.
Following such a min-max principle, GAPGC can reduce redundant information for the main task as much as possible, 
thereby increasing the relevance between the contrastive self-supervised task and the main task. 
Furthermore we provide theoretical evidence that GAPGC can yield upper bound guarantee of the redundant information from the original graphs from information theory perspective.
Empirically, the results on eight different molecular scaffold OOD datasets validate the effectiveness and generalization of our method.

 \textbf{Contributions:}
     1) We propose a graph test time adaptation method GAPGC tailored for GNNs, based on contrastive learning. The GAPGC explores the trained model knowledge by utilizing pseudo-labels in sample selection, and employs a min-max optimization to pull together the self-supervised task and the main task. To the best of our knowledge, it is the first test time adaptation method tailored for GNNs. 
     2) The GAPGC alleviates the \textit{\ConfirmationBias} issue caused by entropy minimization (pseudo-labels) through a group of positive samples. Besides, we also provide theoretical evidence that GAPGC can extract minimal sufficient information for the main task from information theory perspective. 
     3) Experiments on various Scaffold OOD molecular datasets demonstrate that GAPGC achieves state-of-the-art performance for GNN TTA.

\section{Related Work}

\paragraph{Test Time Adaptation.} 

Test time adaptation aims to adapt models based on test samples in the presence of distributional shifts. Test time adaptation can be further subdivided into test time training and fully test time adaptation according to whether it can access the source data and alter the training of the source model. 
Existing test time training methods~\cite{sun2020test,liu2021ttt++} rely on a self-supervised auxiliary task, which is jointly optimized with the main task on the source data and then further finetunes the model on test data. 
Fully test time adaptation methods with only test data contains batch normalization statistics adaptation \cite{li2016revisiting,nado2020evaluating, schneider2020improving}, prediction entropy minimization  \cite{wang2020tent,zhang2021memo,niu2022efficient}, and classifier adjustment~\cite{iwasawa2021test}. 
Our work follows the fully test time adaptation setting and aims to design a TTA method tailored for GNNs. We use adversarial contrastive learning with a group pseudo-positive samples to address two key limitations of prior works (i.e.involving redundant information  and misleading of incorrect pseudo-labels). 

\vspace{-4mm}
\paragraph{Graph Contrastive Learning.}
Recently, contrastive learning (CL) aiming to learn discriminative representation has been widely applied to the visual domain~\cite{tian2020contrastive,chen2020simple,wang2021self}.
In GNNs, many GCL methods are arisen for graph representation learning, such as GraphCL~\cite{you2020graph}, GRACE~\cite{zhu2020deep},  AD-GCL~\cite{suresh2021adversarial} and G-Mixup~\cite{han2022g}.
The performance of GCL heavily relies on the elaborate design of augmentations~\cite{zhao2022graph}. The simple operators like randomly dropping edges or dropping nodes may damage the label-related information and get label-various augmentations~\cite{wang2021towards}.
However, in the self-supervised setting, the dilemma is that model can not directly produce label-invariant augmentations via the current training model ~\cite{guo2022softedge,luo2022automated}. Fortunately, GAPGC uses the decent trained model to generate the relatively reliable pseudo-labels, avoiding the severe model shift caused by the incorrect positive samples.
\begin{table*}[th]
\def\p{$\pm$} 
\setlength\tabcolsep{4pt} 
\centering
\vspace{-2mm}
\caption{Test ROC-AUC (\%) of GIN(contexpred) on molecular property prediction benchmarks with OOD split.('$\uparrow$' denotes  performance improvement compared to the pure test baseline. PF: parameter-free.)}
\vspace{-3mm}
\scalebox{0.83}{
\begin{tabular}{l|cccccccc|c}
\toprule 
Methods & BBBP & Tox21 & Toxcast& SIDER & ClinTox & MUV & HIV & BACE & \textbf{Average}\\
\midrule
\# Test Molecules & 203 & 783 & 857 & 142 & 147 & 9308 & 4112 & 151 & $\diagup$ \\
  \# Binary prediction task & 1 & 12  & 617 & 27 & 2 & 17 & 1 & 1 & $\diagup$\\
\midrule
Test~(baseline)       & 69.23 & \underline{75.44} & 63.68 & 59.70 & 68.94 & 78.32 & 77.52 &  80.16 & 71.62   \\
Tent~\cite{wang2020tent}        & 68.80 & 74.70 & 63.41 & 59.50 & 69.68 & 78.18 & 76.72 &  80.39 & 71.42 \\
BN Ada.~\cite{schneider2020improving} & 69.31 & 75.30&	\underline{63.95} &	60.09&	71.59&	78.63&	77.34&	80.26&72.06
\\
SHOT~\cite{liang2020we} & \underline{69.46} & 74.84 & 63.77 & \underline{60.47} & 68.52 & 79.05 & 77.57 & 80.19 & 71.73\\
\midrule
PF-GAPGC~(ours) & 69.25 & 75.33 & 63.89 & 59.93 & \underline{72.92} & 78.91 & \underline{78.14} & \underline{82.80} & \underline{72.65} \\
GAPGC~(ours) & \textbf{70.34}$\uparrow_{1.11}$ & \textbf{76.00}$\uparrow_{0.56}$  & \textbf{64.58}$\uparrow_{0.90}$	&	\textbf{60.85}$\uparrow_{1.15}$ & \textbf{72.92}$\uparrow_{3.98}$ & \textbf{80.17}$\uparrow_{1.85}$ & \textbf{78.63}$\uparrow_{1.11}$ & \textbf{83.03}$\uparrow_{2.87}$ & \textbf{73.31}$\uparrow_{1.69}$ \\

\bottomrule
    \end{tabular}
    }
    \vspace{-5mm}
    \label{tab:molecular}    
\end{table*}

\begin{figure}[t]
    \centering
    \vspace{-4mm}
    \subfigure{
    \begin{minipage}[s]{0.47\linewidth}
    \centering
    \includegraphics[width=1\linewidth]{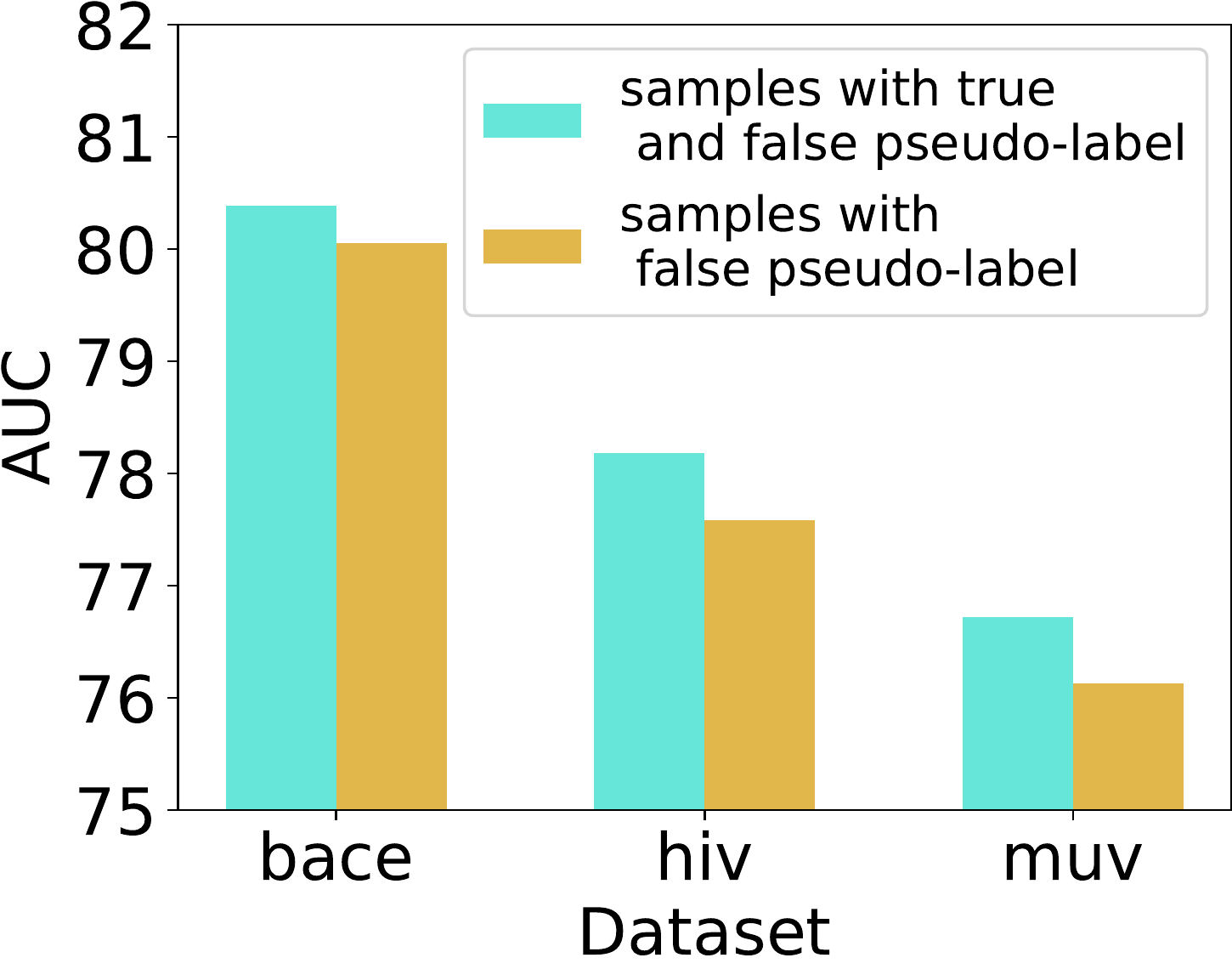}
    \vspace{-7mm}
    \caption*{(a)}
    \vspace{-5mm}
    \end{minipage}
    }
    \subfigure{
    \begin{minipage}[s]{0.47\linewidth}
    \centering
    \includegraphics[width=1\linewidth]{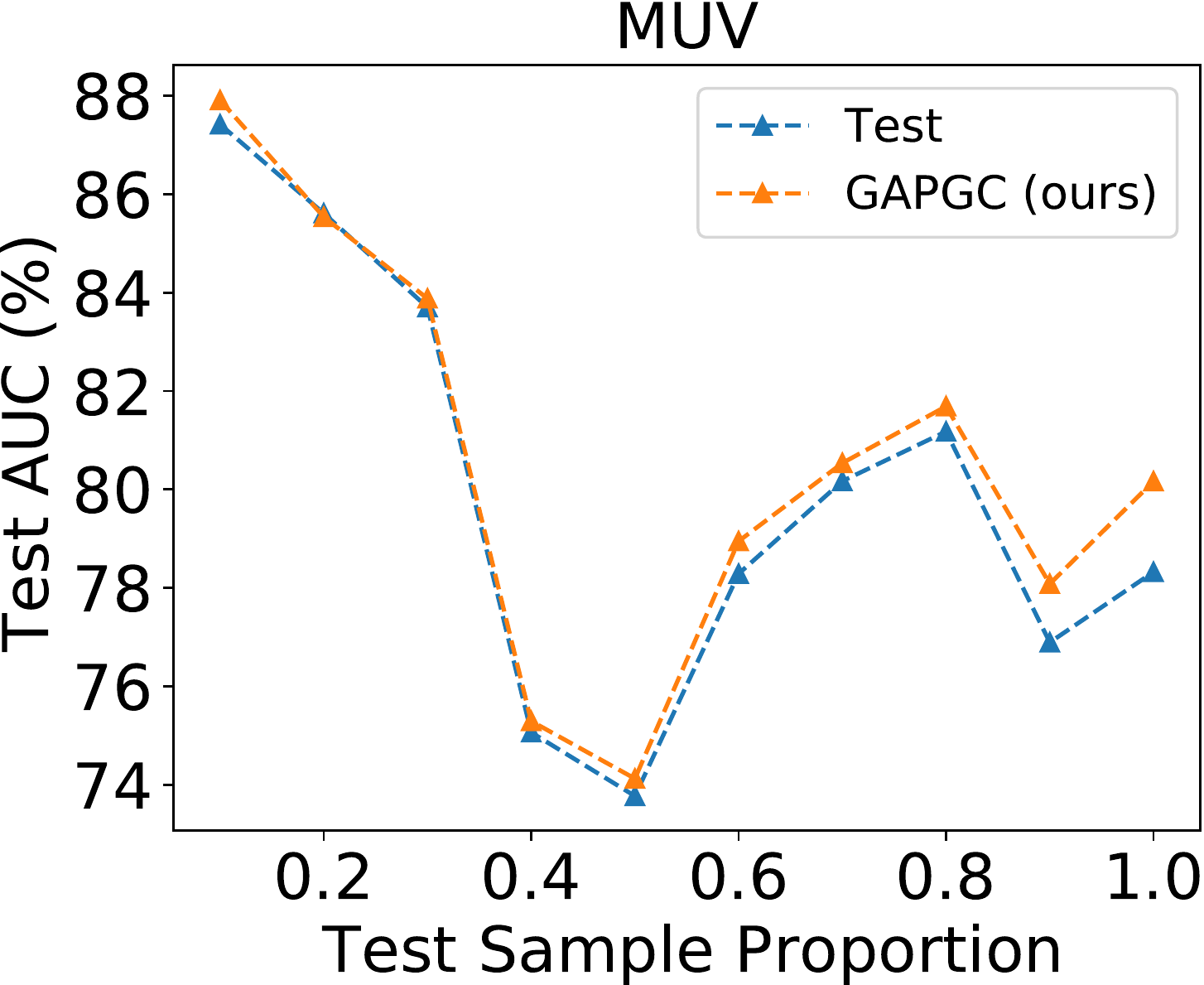}
    \vspace{-7mm}
    \caption*{(b)}
    \vspace{-5mm}
    \end{minipage}
    }
    \vspace{-1mm}
     \caption{(a) Entropy minimization performed on test samples with only false pseudo-label will degrade the performance, compared with that performed on all test samples with both true pseudo-label and false pseudo-label. (b) AUC with different proportion of test data during testing.}
    \label{fig:label_radio}
        \vspace{-7mm}
\end{figure}

\vspace{-1mm}
\section{Graph Adversarial Pseudo Group Contrast}
\vspace{-1mm}


\paragraph{Notations.}  We use boldface letter $\x\in \R^n$ to denote an $n$-dimensional vector, where $\x(i)$ is the $i^\text{th}$ entry of $\x$.  Let $G=(V,E)$ be a graph with vertices $V$ and edges $E$. We denote by $\mathbf{A}\in \R^{|\V|\xx|\V|}$ the adjacency matrix of $G$. Suppose $G$ has node features $\{\x_i\in\R^{F}\}_{i=1}^{|V|}$. 
For test time adaptation, we denote the previously trained model as $f_\Phi$, in which $\Phi$ consists of the learnable parameters of encoder $\Phi_1$ and classifier $\Phi_2$ respectively.

To address those challenges mentioned in the introduction and better adapt the trained model to test dataset online, we propose a new graph test time adaptation method based on a novel Adversarial Pseudo Group Contrast strategy. 
\vspace{-2mm}
\subsection{Group Pseudo-Positive Samples for Graph CL}
\vspace{-1mm}
Here, we aim at adapting the model via utilizing the informative pseudo-labels (i.e. the output of the model) during test time. 
However, as alluded to earlier in the introduction, directly using the entropy minimization would push the probability of a false pseudo-label sample as sharp as possible, thus causing the overconfidence mistake and confusing the model during TTA~(c.f. Fig. \ref{fig:label_radio}(a)). To avoid this problem, we use a Contrastive Learning variant to better explore the pseudo-labels. 
In the CL variant, a number of the same pseudo-class augmentations are selected to mitigate the reliance on pseudo-labels and enhance the tolerance to false labels. 
Specially, for each anchor graph $G_i$ in test set, its representation $\z_{i}=h(f_{\Phi_1}(G_i))$ as well as its augmentation $\z_{i}^{a}=h(f_{\Phi_1}(G_i^{a}))$ are generated by the GNN encoder $f_{\Phi_1}$ following with a projection head $h$ ($2$-layer MLP), where $G_i^a$ is an augmentation of $G_i$ sampling from a given Graph Data Augmentation (GDA) $T(G_i)$.
 Then we obtain the pseudo-label of $G_i$ by forwarding the classifier $f_{\Phi_2}$, i.e. $\hat{y}_i=\text{arg}\max_c{f_{\Phi_2}(  f_{\Phi_1} ( G_i)}) )$.

Instead of using only one positive sample in standard contrastive loss, a group of positive samples $\mcal{S}^p = \{\z_{d}^{a,\hat{y}_i}\}_{d=1}^{D}$ in a mini-batch (suppose the size is $m > D$) are chosen according to the anchor pseudo-label $\hat{y}_i$. In addition, all augmentations with other class $\{\z_i^{a,\hat{y}}:\hat{y}\in\{1,2,\cdots, C\}\setminus\hat{y_i}\}$ are acted as negative samples.  Finally, the contrastive loss variant can be formulated as follows:
\begin{align}
\small
     \hat{\mcal{L}} & (f_\Phi,\{(G_i)\}_{i=1}^{m}, \{G_i^{a}\}_{i=1}^{m}) \notag\\ &:=\frac{-1}{m}\sum_{i=1}^{m}\mcal{F}(\{\log \frac{\exp(s(\z_{i}, \z_{d}^{a,\hat{y}_i}))}{\sum_{j=1}^{m} \exp(s(\z_{i}, \z_{j}^a))}\}_{d=1}^{D}), 
         \label{eq:L_GAPGC}
\end{align}
where $\sum_{j=1}^{m} \exp(s(\z_{i}, \z_{j}^a))$ represents the sum of every pairs of anchor $\z_i$ in a min-batch, $s$ is the cosine similarity, 
and $\mcal{F}:\R^{D}\to \R$ denotes a general multivariate function, such as the mean function. If not specifically stated, we default to $\mcal F$ as the mean function throughout this paper.
It is conspicuous that the loss $\hat{\mcal{L}}$ maximizes the similarity between the anchor graph $G$ and its positive samples $\mcal{S}^p$. When some pseudo-labels in $\mcal{S}^p$ are error, those samples with true pseudo-labels will win this instance conflict since their representations are more similar to the anchor compared with the false ones. 
The group pseudo-positive samples take advantage of pseudo-labels and enhance the tolerance to false labels, effectively alleviating the \textit{\ConfirmationBias}.
It's worth noting that the group positive samples are particularly suitable for graph CL since graph augmentations are always label-sensitive.
For example, in a molecular graph dataset, supposing that all molecular graphs containing a cycle are labeled as toxic. If we drop any node belonging to the cycle, it will damage this cyclic structure, mapping a toxic molecule to a non-toxic one.
\vspace{-2mm}
\subsection{Adversarial Learnable Augmenter for Graph CL}
\vspace{-1mm}
On test time adaptation, the self-supervised task should only concentrate on the performance of the main task. TTT++\cite{liu2021ttt++} can enhance the relevant between the contrastive learning auxiliary task with the main task by jointly optimizing on source data. However, in our setting without access to source data, directly optimizing the contrastive loss $\hat{\mcal{L}}$ is insufficient to obtain a promising performance for the main task~(c.f. Table \ref{tab:ab_ot})  since the contrastive learning would capture the redundant information that is irrelevant to downstream tasks.
This phenomenon has been well studied in early literature under the representation learning context~\cite{tschannen2019mutual,suresh2021adversarial}.
In order to reduce the redundant information, 
we adopt the adversarial learning strategy upon the contrastive learning framework. A trainable GNN-based augmenter optimized by maximizing the contrastive loss is used to decrease the amount of information being encoded.

\paragraph{Learnable Augmenter.}
Here we use a learnable edge-dropping augmentation. For a graph $G$, $T(G)$ is an augmentation of $G$ obtained by dropping some edges from $G$. This can be done by setting the edge weights $e_{ij}$ of $T(G)$ as a binary variable, where edge $(i,j)$ is selected if $e_{ij} =1$ and is dropped otherwise. The binary variable can be regarded as following a Bernoulli distribution with a parameter denoted as $w_{ij}$, where $P(e_{ij}=1)=w_{ij}$. We parameterize the Bernoulli parameter $w_{ij}$ with a GNN-based augmenter initialized from the offline trained model. Further, to make the augmenter trainable, we relax the edge weights to continue variable in (0,1) with the reparameterization trick~\cite{maddison2016concrete}. Specifically, the weight $\hat{e}_{ij}$ of edge $(i,j)$ is calculated by:
\begin{align}
    & \omega_{ij}=\text{MLP}([\mbf{h}_i,\mbf{h}_j]),\quad 
    \delta \backsim \text{Uniform}(0,1) \label{eq:mlp} \\ 
    \hat{e}_{ij}& =\text{Sigmoid}((\log\delta-\log(1-\delta)+\omega_{ij})/\tau) 
\end{align}
where $\mbf h_i$ is the $i$-th node representation output by GNN encoder $g_{\Theta_1}(G)$, and $\Theta$ corresponds to the learnable parameters of GNN encoder ($\Theta_1$) as well as MLP ($\Theta_2$). Note that $e_{ij}$ is approximated to binary when the temperature hyper-parameter $\tau\to 0$.  Besides, the regularization term $\mcal{L}_e=\sum_{ij}\omega_{ij}/|E|$ is also added to prevent the excessive perturbation.

\paragraph{Min-Max Game.}

The GAPGC TTA framework is constructed following a min-max principle:  
\begin{align}
    \small
   & \max\limits_{\Theta}\min\limits_{\Phi}{\mcal{L}}_{GAPGC}(f_{\Phi},g_{\Theta},\{G_i\}_{i=1}^{m})=  \notag \\
   & \max\limits_{\Theta}\min\limits_{\Phi}(\hat{\mcal{L}}(f_\Phi, \{G_i\}_{i=1}^{m}, \{g_{\Theta}(G_i)\}_{i=1}^{m})
   + \lambda \mcal{L}_e)
    \label{eq:min_max_L}
\end{align}
where $\lambda$ is the regularization weight. 
This objective contains two folds: (1) Optimize encoder $f_{\Phi_1}$ to pull an anchor and a number of pseudo-positive samples together in the embedding space, while pushing the anchor away from many negative samples; (2) Optimize the the augmenter $g_{\Theta}$ to maximize such a contrastive loss. 
Overall, the min-max optimization expects to train an encoder that is capable of maximizing the similarity between the original graph and a set of augmentations even though the augmentations are very aggressive~(i.e. the $T(G)$ is clearly different with $G$).

The interpretation from information bottleneck perspective shows that the objective can provide a lower bound guarantee of the information related to the main task, while simultaneously holding a upper bound guarantee of the redundant information from the original graphs, with the details can be seen in Appendix \ref{sec:theorical} due to the space limitation. Finally, this objective can get representations containing minimal information that is sufficient to identify each graph. 

\vspace{-2mm}

\paragraph{Discussion.}
 a) A recent approach called PGC (Pseudo Group Contrast )~\cite{wang2021towards} shares a similar manner with our group pseudo-positive samples for GCL. However, they are different in the following aspects: (1) GAPGC aims at working for TTA (self-supervised) while PGC is designed for self-tuning (self-training + fine-tuning), and the former explores the test data without any labels while the latter needs both labeled and unlabeled data.  (2) GAPGC uses adversarial learning to generate the hard positive samples and reduce the redundant information between input graphs and graph representations while PGC is only designed for image data with simple augmentations.
b) The style of learnable edge-dropping augmentation has also been used in representation learning AD-GCL \cite{suresh2021adversarial}. The difference between our method with AD-GCL is that (1) we parameterize the Bernoulli weights with a GNN-based augmenter initialized from the offline trained model, while AD-GCL uses another GNN trained from scratch, and (2) our augmenter is optimized by a contrastive loss variant with group pseudo-positive samples, different from AD-GCL.

 \begin{table*}[th]
\def\p{$\pm$}

\setlength\tabcolsep{16pt}
\centering
\vspace{-1mm}
\caption{ The results of ablation study on GAPGC(ROC-AUC \%).  ALA means Adversarial Learnable Augmenter and GPPS represents Group Pseudo-Positive Samples.}
\vspace{-3mm}
\scalebox{0.8}{
\begin{tabular}{l|cccccccc|c}
\toprule 
Methods & BBBP & Tox21 & Toxcast& SIDER   & ClinTox  & MUV & HIV & BACE & \textbf{Average}\\
\midrule
test (baseline) & 69.23 & \underline{75.44} & 63.68 & 59.70) & 68.94 & 78.32 & 77.52 & 80.16 & 71.62\\
w/ Both & \textbf{69.98} & \textbf{75.72} & 64.18 & \textbf{60.21} & \textbf{71.79} & \underline{80.17} & \textbf{78.32} & \textbf{82.28} & \textbf{72.83}\\
 w/o ALA & 68.90 & 75.39 & \underline{64.42}  & \underline{59.85}   & 70.98 & \textbf{82.50} & \underline{77.67} & 81.05 & \underline{72.60}\\
w/o GPPS & \underline{69.94} & 75.40 & 64.41 & 59.70  & \underline{71.25}  & 76.94 & 77.31 & \underline{81.97} &72.12\\
w/o Both & 68.87 & 75.38 & \textbf{64.47} & 59.78  & 70.91 & 80.09 & 77.57 & 80.93 &72.25\\
\bottomrule
    \end{tabular}
    }
    \vspace{-5mm}
    \label{tab:ab_ot}    
\end{table*}

\section{Experiment}
\vspace{-1mm}
We conduct experiments on molecular scaffold OOD datasets to evaluate our method. 
The multi-task style version is deferred to the appendix due to the space limitation.
\vspace{-4mm}
\subsection{Molecular Property Prediction}
\vspace{-2mm}
\paragraph{Settings.} We use the same model architecture and datasets in~\citet{hu2019strategies}. \textbf{Dataset:} eight binary classification datasets in MoleculeNet~\cite{wu2018moleculenet} is used and the data split follows the OOD split principle: \underline{scaffold split}.
The split ratio for the train/validation/test sets is $8$:$1$:$1$. \textbf{Architecture:}
A 5-layer GIN~\cite{xu2018powerful}.
We first directly initialize the model with a pretrained model GIN~(contextpred) released at ({\url{https://github.com/snap-stanford/pretrain-gnns}}), which is pretrained via the \textit{Context Prediction} on Chemistry dataset. Then the model would be trained on the training set offline and finally adapted on the testing set online to evaluate TTA methods.  More details of the experimental settings can be referred to the appendix .
\vspace{-3mm}
\paragraph{Baselines.} Since we have not found related works about test time adaptation on graph data, we extend several state-of-the-art baselines designed for Convolutional Neural Networks to GNNs: Tent, BN~Adaptation, and SHOT. Specifically, Tent \citep{wang2020tent} minimizes the entropy of the model predictions on test data during testing. BN~\cite{schneider2020improving} updates the batch normalization statistics according to the test samples. SHOT~\cite{liang2020we} exploits both information maximization and self-supervised pseudo-labeling during testing.
\vspace{-2mm}
\paragraph{PF-GAPGC: Parameter-Free GAPGC.}
In test time training, since it is not able to access the training data, the projection head $h$ and  MLP $g_{\Theta_2}$ can not get an ideal weights initialization when the number of test data points is extremely small. To address this problem, we extend GAPGC to a parameter-free version, i.e. using only the weights of the trained model and not introducing any extra learnable parameters. 
\textbf{i) Remove Projection Head.}
Although many works have shown that the projection head is vital for improving the expressive power of representations ~\cite{chen2020simple, luo2022automated}, it will be better to remove it when we don't have enough data to train~\cite{chen2022contrastive}. 
\textbf{ii) Parameter-Free $g_{\Theta_2}$.}
We directly replace MLP $g_{\Theta_2}$ in \Eqref{eq:mlp} with a inner product for removing the extra parameters $\Theta_2$, i.e.
\begin{align}
    \omega_{ij}=\frac{\mbf{h}_i^{\top}\mbf{h}_j}{\|\mbf h_i\|_2\|\mbf h_j\|_2}
\end{align}
where $\|\cdot\|$ is $L_2$ norm.
Note that after this replacement,  $g_{\Theta}$ can be seen as a graph inner product decoder variant.
\vspace{-2mm}
\paragraph{Results.} The results compared with different TTA methods are shown in Table~\ref{tab:molecular}. 
Obs. (1): Both PF-GAPGC and GAPGC improve over all the baselines on average, and GAPGC consistently gains the best performance among all the baselines in all different datasets, suggesting its generalization and robustness. And since  GAPGC is better than PF-GAPGC, it demonstrates that the contrastive projector $h$ and MLP in \textit{augmenter} $g_{\Theta}$ can learn well even though they have not been pre-trained offline.
Obs. (2): The performance of Tent is worse than the test baseline on average, indicating that directly using entropy minimization may lead to negative transfer. This may be caused by optimizing the entropy is biased to predict only a particular class, which will hurt the performance in the binary classification problem, as illustrated in Fig.  \ref{fig:label_radio}(a). Besides, SHOT only gives a small performance improvement, further implying the limitations of entropy minimization for this binary classification task on graphs.
\vspace{-1mm}
\subsection{Ablation Study}
\vspace{-1mm}
\paragraph{1) Effects of Different Components.} We conduct ablation experiments to verify the efficiency of the proposed two components: Adversarial Learnable Augmenter (ALA) and Group Pseudo-Positive Samples (GPPS). 
 The results are shown in Table~\ref{tab:ab_ot}, where without Adversarial Learnable Augmenter means randomly edge-dropping augmentation is used instead, while without Group Pseudo-Positive Samples represents only the correspondent augmentation of the anchor graph is used as its positive sample in the contrastive loss.
  Obs. (1): The four models all outperform the test baseline, verifying that contrastive learning as a self-supervised task during testing can improve test performance.  Obs. (2): The variant 'w/o ALA' is better than the variant 'w/o GPPS', demonstrating that the Adversarial Learnable Augmenter contributes more in GAPGC. Obs.~(3): The combination of these two components achieves the best performance on average.

\vspace{-2mm}

 \paragraph{2) Effects of Different Proportion of Test Samples.} In actual deployment, test data may come in batches of different sizes. Here we evaluate the effectiveness of GAPGC under different proportions of test data. As described in Fig.~\ref{fig:label_radio}(b), our GAPGC can consistently gain performance improvement compared with the pure test baseline under different proportions of test data.

\vspace{-2mm}
\section{Conclusions and Limitations}
We propose a graph test time adaptation strategy \textbf{GAPGC} for the graph OOD problem.
Despite the good results, there are some aspects worth exploring further in the future:
i) The parameterized augmenter can generalize to other learnable graph generators. ii) The GCL projector and augmenter's MLP can be initialized offline by training on public datasets (e.g., Chemistry dataset) prior to deployment.



\bibliography{0_main}
\bibliographystyle{icml2022}





\newpage
\onecolumn
\appendix
\begin{center}
\Large
\textbf{Appendix}
 \\[20pt]
\end{center}

\etocdepthtag.toc{mtappendix}
\etocsettagdepth{mtchapter}{none}
\etocsettagdepth{mtappendix}{subsection}

{\small \tableofcontents}

\section{GAPGC for Multi-tasks}
\subsection{Multi-Tasks }
\paragraph{Multi-labels Similarity.}
We define a similarity between two multi-labels for selecting the positive samples in a min-batch. 

\begin{definition}
Given the label of two n-task samples   $y_1=(t_1,t_2,\cdots,t_n)$, $y_2=(l_1,l_2,\cdots,l_n)$, where $t_i\in \{0,1\}$and $l_i\in \{0,1\}$ for all $i\in\{1,2,\cdots, n\}$. The similarity between $y_1$ and $y_2$ can be defined as

\begin{align}
    sim(y_1, y_2)  = \frac{1}{n}\sum_{i=1}^n\mathbf{1}(t_i=l_i)
\end{align}
\end{definition}

\paragraph{Positive Sample Selection.}
In GAPGC, suppose that we get the pseudo-label of anchor is $\hat{y}=[y_1,y_2,\cdots,y_n]$ and its augmentation pseudo-label is $\hat{y}^{a}=[y_1^a,y_2^a,\cdots,y_n^a]$.
The similarity can be calculated by the $sim(\hat{y}, \hat{y}^{a})$. 

Since $sim \in [0,1]$ is a real number, we use a threshold $\gamma\in [0,1]$ to decide whether the  the augmentation graph can be selected as positive sample, i.e. 
\begin{align}
    \mathcal{Y}(\hat{y}, \hat{y}^{a}) =
\begin{cases}
	0 , & \text{if } sim(\hat{y}, \hat{y}^{a}) < \gamma\\
	1 , & \text{if } sim(\hat{y}, \hat{y}^{a}) \geq \gamma
\end{cases}
\end{align}

$\mathcal{Y}(\hat{y}, \hat{y}^{a})=1$ represents the augmentation is selected otherwise is not selected.

\section{Theoretical Analysis From Graph Information Bottleneck}\label{sec:theorical}
In this part, we show the GAPGC is actually very closely linked to the GIB~\cite{wu2020graph}, which minimizes the mutual information between input graph and graph representation and simultaneously maximizes the mutual information between representation and output, i.e.$\min_f[I(f(G);G)-\beta I(f(G);Y)]$.

Different from the existing works AD-GCL~\cite{suresh2021adversarial} that just provides the connection between standard adversarial contrastive loss and GIB,
here we give a more general version based on GAPGC.

\paragraph{Min-Max Mutual Information.}
We consider $\F$ as a general multivariate function to investigate a comprehensive conclusion about the GAPGC and mutual information. The results show below.

\begin{restatable}{proposition}{restadePropoone}
\label{Pro:mutual_inoformation}

Let $\z_{i}=h(f_{\Phi_1}(G_i))$ and $\z_{d}^{a, \hat{y}_i}=h(f_{\Phi_1}(G_d^{a,\hat{y}_i}))$ denote the representation of anchor graph $G_i$ and a positive augmentation $G_d^{a,\hat{y}_i}$, respectively.
where $\hat{y}_i={arg}\max_c{f_{\Phi_2} ( f_{\Phi_1} ( g_{\Theta}(G_i)}))$ is the pseudo-label of $G_i$.
If $\mcal{F}$ is a multivariate function satisfying

\quad \textcolor{blue}{ i) enables exchange order with expectation,} and

\quad  \textcolor{blue}{ii) monotonically non-decreases for each component,} 

then the contrastive loss in \Eqref{eq:L_GAPGC} can be bounded by

\begin{align}
    \mcal{F}(\{I(\z_i,\z_{d}^{a, \hat{y}_i})\}_{d=1}^{D}) \geq \F(\log(m)) - \hat{\mathcal{L}}_{GAPGC}
\end{align}
where $g_{\theta}(G)$ is the parameterized augmenter, composed of a GNN encoder $g_{\Theta_0}$ 
and a MLP decoder $g_{\Theta_1}$. 

\end{restatable}
Therefore, the \Eqref{eq:min_max_L} is approximately equivalent to optimizing the transformation~($\mcal F$ is the transform function) of mutual information between the anchor graph and its positive samples. Formally, we have
\begin{align}
   {\mcal{L}}_{GAPGC}  \thickapprox \min\limits_{\Theta}\max\limits_{\Phi}\mathcal{F}\left( \{I(f_{\Phi}(G_i);f_{\Phi}(g_{\Theta}(G_d^{a,\hat{y}_i})))\}_{d=1}^{D}\right)
\end{align}
Note that the mean function is a multivariate function that satisfies the conditions in Proposition \ref{Pro:mutual_inoformation}.
For clarity, we rewrite $g_{\Theta}$ as a graph data augmentation family $\mcal T$ and omit the subscript $\Theta$ of $f_{\Theta}$. We have
\begin{align}
   {\mcal{L}}_{GAPGC}  \thickapprox \min\limits_{T\in \mcal T}\max\limits_{f}\mathcal{F}\left( \{I(f(G);f(t(G_d^{a,\hat{y}_i})))\}_{d=1}^{D}\right), \text{where } t(G)\backsim T(G)
\end{align}
\paragraph{Graph Information Bottleneck~(GIB).} We theoretically describe the property of the encoder online trained via GAPGC and explain it by GIB. 

\begin{restatable}{proposition}{restadePropotwo}
\label{pro:GIB}
Assume that the GNN encoder $f$ has the same power as the 1-WL test and $\mcal F$ satisfies the conditions in proposition ~\ref{Pro:mutual_inoformation}. Suppose
$\G$ is a countable space and thus  quotient space $\G'= G/\cong$ is a countable space, where $\cong$ denotes equivalence ($G_1\cong G_2$ if
$G_1, G_2$ cannot be distinguished by the 1-WL test). Define $\mbb{P}_{\mcal{G}'\times\mcal{Y}}(G',Y')=\mbb{P}_{\mcal{G}\times\mcal{Y}}(G\cong G', Y)$ and $T'(G')=\mbb{E}_{G\backsim \mbb{P}_{\G}}[T^*(G)|G\cong G']$ for $G'\in \mcal{G}$. Then, the optimal solution $(f^*, T^{*} )$ to
GAPGC satisfies

   1. $\mathcal{F}\left( \{I(f^*(t^*(G_d));G|Y)\}_{d=1}^{D}\right)\leq \min_{T\in \mcal{T}} \mathcal{F}(\{ I(t'(G'_d);G')-I(t'^*(G'_d);Y)\}_{d=1}^{D})$, 
   where $t^*(G_d)\backsim T^*(G), t'(G_d')\backsim T'(G'), t'^*(G_d')\backsim T'^*(G'), (G, Y)\backsim \mbb P_{\mcal{G}\times \mcal{Y}}$.
   
  2. $\mathcal{F}\left( \{I(f^*(G);Y)\}_{d=1}^{D}\right)\geq \mathcal{F}\left(\{I(f^*(t'^*(G'_d));Y)\}_{d=1}^{D}\right)=\mathcal{F}\left(\{I(t'^*(G_d');Y)\}_{d=1}^{D}\right)$,
  where $t'^*(G_d')\backsim T'^*(G'), (G, Y)\backsim \mbb{P}_{\mcal{G}\times\mathcal{Y}}$ and $(G',Y)\backsim\mbb{P}_{\mcal{G}'\times\mcal{Y}}$.
\end{restatable}

The left side of inequality in statement 1 measures the redundant information that is embedded in representations but irrelevant to the main tasks during TTA. i) The result in statement 1 of proposition \ref{pro:GIB} suggests that the redundant information have a GIB style upper bound, since $\min_{T\in \mcal{T}}I(t'(G');G')-I(t'^*(G');Y)\geq \min_f[I(f(G);G)-I(f(G);Y)]$ (i.e. $\beta=1$ of GIB). Therefore, the encoder trained by GAPGC enables encoding a representation that only has limited redundant information.

ii) The statement 2 shown that $\mathcal{F}\left(\{I(t'^*(G_d');Y)\}_{d=1}^{D}\right)$ can act as a lower bound of the mutual information between the learnt representations and the labels of the main tasks. Therefore, optimizing the GDA family $\mathcal{T}$ allows $\mathcal{F}\left( \{I(f^*(t^*(G_d));Y)\}_{d=1}^{D}\right)$ to achieve a larger value.
This implies that it is better to regularize when learning over $\mathcal{T}$ . In the GAPGC loss, based on
edge-dropping augmentation, we follow ~\cite{suresh2021adversarial} to regularize the ratio of dropped edges per graph.

When $\F$ is a mean function, the statement 2 can be rewritten as:
\begin{align}
    I(t'^*(G_d');Y) \geq \frac{1}{D}\sum_{d=1}^D I(f^*(t'^*(G'_d));Y) 
\end{align}

\section{Missing Proof.}
\restadePropoone*
\begin{proof}
From \citet{oord2018representation}, we easily get a lower bound of mutual information between $z_i$ and $z_j^a$.
\begin{align}
    \mbb{E}_{(G,G^{\hat{y},d})}\left[\log \frac{\exp(s(\z_{i}, \z_{d}^{a,\hat{y}_i})}{\sum_{j=1}^{m} \exp(s(\z_{i}, \z_{j}^{a}))}\right]
    &=
    \frac{1}{m}\sum_{i=1}^{m}\log \frac{\exp(s(\z_{i}, \z_{d}^{a,\hat{y}_i})}{\sum_{j=1}^{m} \exp(s(\z_{i}, \z_{j}^{a}))} \\
    & \leq I(\z_i,\z_{d}^{a,\hat{y}_i})-\log(m)
\end{align}
\begin{align}
    \hat{\mcal{L}}_{GAPGC}:=-\mbb{E}_{(G,G^{\hat{y}})}\left[\mcal{F}(\{\log \frac{\exp(s(\z_{i}, \z_{d}^{a,\hat{y}_i})}{\sum_{j=1}^{m} \exp(s(\z_{i}, \z_{j}^{a}))}\}_{d=1}^{D})\right]
    =-\frac{1}{m}\sum_{i=1}^{m}\mcal{F}(\{\log \frac{\exp(s(\z_{i}, \z_{d}^{a,\hat{y}_i})}{\sum_{j=1}^{m} \exp(s(\z_{i}, \z_{j}^{a}))}\}_{d=1}^{D})
\end{align}
where $\mcal{F}: \R^{d}\to \R$ is a general multivariate function, which can be designed a learnable neural network such as MLP, Attention networks.

Since $\mcal{F}$ can exchange order with expectation, then we have
\begin{align}
    \mbb{E}_{(G,G^{\hat{y}})}\left[\mcal{F}(\{\log \frac{\exp(s(\z_{i}, \z_{d}^{a,\hat{y}_i})}{\sum_{j=1}^{m} \exp(s(\z_{i}, \z_{j}^{a}))}\}_{d=1}^{D})\right] 
    & = \mcal{F}(\{\mbb{E}_{(G_i,G_i^{\hat{y}})}\left[\log \frac{\exp(s(\z_{i}, \z_{d}^{a,\hat{y}_i})}{\sum_{j=1}^{m} \exp(s(\z_{i}, \z_{j}^{a}))}\right]\}_{d=1}^{D} )\\
    & \leq  \mcal{F}(\{I(z_i,\z_{d}^{a,\hat{y}_i})-\log(m)\}_{d=1}^{D})
\end{align}

If $\mcal{F}$ is monotonically non-decreasing for each component, minimizing $\hat{\mcal{L}}_{GAPGC}$ is equivalent to maximize the lower bound of $\mathcal{F}\left( \{I(f_{\Phi_1}(G);f_{\Phi_1}(g_{\Theta}(G_d^{\hat{y}})))\}_{d=1}^{D}\right)$.

In our experiments, we use the mean function as an instance.

\end{proof}

\restadePropotwo*
\begin{proof}
Our proof based on the lemma below.
\begin{lemma}[\cite{suresh2021adversarial}, theorem 1]
\label{le:gib}
Suppose the encoder $f$ is implemented by a GNN as powerful as the 1-WL test. Suppose $\mcal G$ is a countable space and thus $\mcal G'$ is a countable space. Then, the optimal solution $(f^*
, T^*)$ to $\min_{T\in \mcal{T}} \max_{f} I(f(G),f(t(G)))$ satisfies, letting $T'^{*} (G') = \mbb{E}_{G\backsim \mbb{P}_{G}} [T^*(G)|G \cong G']$,

1. 
\begin{align}
\label{eq:gib1}
    I(f^*(t^*(G)); G | Y ) \leq \min_{T \in \mcal T} I(t'(G'); G'))- I(t'^*(G'); Y ),
\end{align} 
where $t^*(G) \backsim T^*(G), t'(G') \backsim T'(G'), t'^*(G') \backsim T'^*(G'), (G, Y ) \backsim \mbb{P}_{\mcal G\times \mcal Y}$ and $(G', Y ) \backsim \mbb{P}_{\mcal G'\times \mcal Y }.
$

2. $I(f^*(G); Y ) \geq I(f^*(t'^*(G'
)); Y ) = I(t'^*(G'
); Y )$, where $t^*(G) \backsim T^*(G), t'^*(G') \backsim T'^*(G'), (G, Y ) \backsim \mbb{P}_{\mcal G\times \mcal Y}$ and $(G', Y ) \backsim \mbb{P}_{\mcal G'\times \mcal Y} $.
\end{lemma}

In lemma \ref{le:gib}, we know for each anchor-augmentation pair, the mutual information between them can be bounded by $ I(f^*(t^*(G_d)); G | Y ) \leq \min_{T \in \mcal T} I(t'(G_d'); G'))- I(t'^*(G_d'); Y ), d=1,2,\cdots D$. So we use the $\mcal F$ to act on two sides of \Eqref{eq:gib1}, then we get
\begin{align}
        \F(\{I(f^*(t^ *(G_d)); G | Y )\}_{d=1}^D) \leq \F (\{\min_{T \in \mcal T} I(t'(G_d'); G'))- I(t'^*(G_d'); Y )\}_{d=1}^D),
\end{align}
Since the $\F$ is monotonically non-decreases for each component, we have $\F (\{\min_{T \in \mcal T} I(t'(G_d'); G'))- I(t'^*(G_d'); Y )\}_{d=1}^D) \leq \min_{T \in \mcal T}( \F( \{I(t'(G_d'); G'))- I(t'^*(G_d'); Y )\}_{d=1}^D)$. Thus,
\begin{align}
    \F(\{I(f^*(t^*(G_d)); G | Y )\}_{d=1}^D) \leq \min_{T \in \mcal T}( \F( \{I(t'(G_d'); G'))- I(t'^*(G_d'); Y )\}_{d=1}^D)
\end{align}

Similarly, the statement 2  can be obtained due to the monotonicity of $\F$
\begin{align}
    \F(\{I(f^*(G); Y )\}_{d=1}^D) \geq  \mathcal{F}\left(\{I(f^*(t'^*(G'_d));Y)\}_{d=1}^{D}\right)=\mathcal{F}\left(\{I(t'^*(G_d');Y)\}_{d=1}^{D}\right)
\end{align}

\end{proof}

\section{Details of Experiments.}
\paragraph{Settings.} In our experiments, we fix the weights of classifier $f_{\Phi_2}$ for alleviating the over-fitting and Catastrophic Forgetting in TTA~\cite{niu2022efficient} and improving the model generalization.

\textbf{Baselines:} Tent \citep{wang2020tent} minimizes the entropy of model prediction during testing. BN~\cite{schneider2020improving} updates the batch normalization statistics according to the test samples. SHOT~\cite{liang2020we} exploits both information maximization and self supervised pseudo-labeling to implicitly align representations from the target domains to the source hypothesis.

\textbf{Datasets:} The eight molecular datasets used are summarized in Table~\ref{tab:dataset}.
 
 \begin{table}[ht]
\def\p{$\pm$} 
\centering
\vspace{-2mm}
\caption{Statistical information of the datasets. We provide (median, max, min, mean, std) number of graph vertices in different datasets.}
\vspace{-0mm}
\scalebox{0.95}{
\begin{tabular}{l|rr|rrrrr}
\toprule 
dataset &\# Molecules & \# training set & median & max & min & mean & std \\
\midrule
 BBBP & 2039 &1631 & 22 & 63 & 2 & 22.5 & 8.1 \\
 Tox21 & 7831& 6264 & 14 & 114 & 1 & 16.5 & 9.5 \\
Toxcast & 8575 & 6860 & 14 & 103 & 2 & 16.7 & 9.7 \\
SIDER & 1427& 1141 & 23 & 483 & 1 & 30.0 & 39.7 \\
ClinTox & 1478 & 1181 & 23 & 121 & 1 & 25.5 & 15.3 \\
MUV & 93087 & 74469 & 24 & 44 & 6 & 24.0 & 5.0 \\
HIV & 41127 & 32901 & 23  & 222 & 2 & 25.3 & 12.0 \\
BACE & 1513 & 1210 & 32 & 66 & 10 & 33.6 & 7.8 \\
\bottomrule
    \end{tabular}
    }
    \vspace{-0mm}
    \label{tab:dataset}    
\end{table}

\textbf{Datasets Splitting:} We follow the \textit{Out Of Distribution} data split principle in \citet{hu2019strategies}: \textit{scaffold split}, in which, molecules are clustered by scaffold (molecular graph substructure), and then the dataset is split so that different clusters end up in the training, validation and test sets.  

\textbf{Hyper-parameter Strategy:}
We use Adam optimizer with $L_2$ regularization during test time. For all hyper-parameters, we use grid search strategies  and the range of hyper-parameters listed in Table~\ref{tab:hyper_parameter}, \ref{tab:hyper_parameter_baselines}, where Table \ref{tab:hyper_parameter} shows our methods and Table \ref{tab:hyper_parameter_baselines} shows the baselines. 
\begin{table}[htbp]
\centering

\caption{Hyper-parameter search range for GAPGC. }
\vspace{-1mm}
\resizebox{0.5 \linewidth}{!}{
\begin{tabular}{c|l|c}
\toprule
      & Hyper-parameter        & Range               \\ 
\midrule
    & $\lambda$        & \{0.1, 1, 10, 20, 50\}                  \\\
    & $\gamma$        & \{0.6, 0.7, 0.8, 0.9, 1.0\}                  \\\
    & Learning rate        & \{0.001, 0.005, 0.01,0.0005,0.0001\}                  \\\
    & Weight decay & \{1e-2, 1e-3,1e-4,5e-4, 1e-5,1e-6,1e-7,0\}  \\
    & Dropout rate        & \{0,0.05,0.1,0.15,0.2,0.25,0.3,0.35, 0.4,0.45, 0.5\} \\
    & Batch size & \{32,64,128,256,512\}  \\
    \midrule
    & Optimizer & Adam  \\
    & TTA Epoch & 1 \\
    & runseed & 1\\
    & GPU & 
Tesla V100 \\
  
\bottomrule                     
\end{tabular}                   \textit{}        
}
\label{tab:hyper_parameter}
\end{table}

\begin{table}[htbp] 
\centering

\caption{Hyper-parameter search range for Baselines. }
\vspace{-1mm}
\resizebox{0.5 \linewidth}{!}{
\begin{tabular}{c|l|c}
\toprule
      & Hyper-parameter for Tent       & Range               \\ 
\midrule

    & Learning rate        & \{0.001, 0.005, 0.01,0.0005,0.0001\}                  \\\
    & Weight decay & \{1e-2, 1e-3,1e-4,5e-4, 1e-5,1e-6,1e-7,0\}  \\
    & Batch size & \{32,64,128,256,512\}  \\    
    & Optimizer & Adam \\
 \midrule  
 \midrule
    & Hyper-parameter for BN       & Range               \\ 
    \midrule
    & Batch size & \{32,64,128,256,512\}  \\  
\midrule
\midrule
    & Hyper-parameter for SHOT       & Range               \\ 
    \midrule
     & Learning rate        & \{0.001, 0.005, 0.01,0.0005,0.0001\}                  \\\
    & Weight decay & \{1e-2, 1e-3,1e-4,5e-4, 1e-5,1e-6,1e-7,0\}  \\
    & $\beta$ & \{0 0.1, 0.2, 0.3,0.4,0.5,0.7,0.8,0. 1\} \\
    & Batch size & \{32,64,128,256,512\}  \\    
    & Optimizer & Adam \\

\bottomrule                     
\end{tabular}                   \textit{}        
}
\label{tab:hyper_parameter_baselines}
\end{table}

\section{Additional Experiments.}

We conduct experiments on molecular scaffold OOD datasets with the non pre-trained GIN model, that is, directly training GIN on the training set of OOD datasets without weights initialization from the pretrained model provided by \citet{hu2019strategies}. The results show in Table~\ref{tab:molecular_nopre}.
\begin{table*}[th]
\def\p{$\pm$} 
\setlength\tabcolsep{4pt} 
\centering
\vspace{-0mm}
\caption{Test ROC-AUC (\%) of GIN on molecular property prediction benchmarks with OOD split.('$\uparrow$' denotes  performance improvement compared to the pure test baseline. PF: parameter-free.)}
\vspace{-3mm}
\scalebox{0.83}{
\begin{tabular}{l|cccccccc|c}
\toprule 
Methods & BBBP & Tox21 & Toxcast& SIDER & ClinTox & MUV & HIV & BACE & \textbf{Average}\\
\midrule
\# Test Molecules & 203 & 783 & 857 & 142 & 147 & 9308 & 4112 & 151 & $\diagup$ \\
  \# Binary prediction task & 1 & 12  & 617 & 27 & 2 & 17 & 1 & 1 & $\diagup$\\
\midrule
Test~(baseline)       & 64.56 & 72.57 & 63.72 & 57.81 & 65.54 & 71.06 & 72.85 &  74.54 & 67.83   \\
Tent~\cite{wang2020tent}        & 65.17 & 72.46 & 63.83 & 60.35 & 70.66 & 71.28 & 71.79 & 74.84 & 68.80 \\

BN Ada. & 64.40 & 72.71 & 63.88 & 60.58 & 73.40 & 71.23 & 71.57 &  74.89 & 69.07\\
SHOT~\cite{liang2020we} & 65.99 & 72.55 & 63.91 & 58.21 & 68.36 & 72.19 & 73.43 & 75.60 & 68.78\\
\midrule

GAPGC~(ours) & 68.27$\uparrow_{3.71}$ &	73.21$\uparrow_{0.64}$ &	63.29$\downarrow_{0.43}$ &	58.57$\uparrow_{0.76}$ &	69.91$\uparrow_{4.37}$ &	71.43$\uparrow_{0.37}$ &	73.12$\uparrow_{0.27}$ &	75.27$\uparrow_{0.73}$&	69.13$\uparrow_{1.30}$ \\

\bottomrule
    \end{tabular}
    }
    \vspace{-7mm}
    \label{tab:molecular_nopre}    
\end{table*}


\end{document}